\documentclass[11pt]{article}
\usepackage[margin=0.9in]{geometry}

\usepackage[round]{natbib}
\usepackage{amsmath,amssymb,amsthm}
\usepackage{bm}
\usepackage{xcolor}
\usepackage{caption}

\usepackage{shortcuts} 

\usepackage[titlenumbered,linesnumbered,ruled,noend]{algorithm2e}

\SetCommentSty{mycommfont}
\SetEndCharOfAlgoLine{}

\usepackage[textsize=tiny]{todonotes}
\title{On sparsity, extremal structure, and monotonicity properties of Wasserstein and Gromov-Wasserstein optimal transport plans}
\author{
Titouan Vayer\\
\small Inria, Rennes, France.
}

\date{}

\begin{document}

\maketitle

\begin{abstract}
This note gives a self-contained overview of some key properties of the
Gromov-Wasserstein (GW) distance, compared with the standard linear optimal transport
(OT) framework. More specifically, I explore the following questions: are GW optimal transport plans sparse? under what conditions does there exist an optimal plan supported on a permutation? do optimal plans satisfy a form of cyclical monotonicity?
In particular, I present the conditionally negative/positive definite (CND/CPD)
properties and show that, when they hold, there exist GW optimal plans that are sparse
and supported on a permutation. In the last section, I also give examples where no
permutation is optimal when the CND or CPD assumption is dropped.
\end{abstract}

\section{Introduction}

This note grew out of discussions with colleagues, in which I realized that a simple,
self-contained exposition of the fundamental properties of Gromov-Wasserstein (GW)
optimal plans was missing. The aim is not to present fundamentally new results, but to
highlight a few properties of GW that I find particularly interesting. Although these
results appear, in one form or another, in the literature, they are rarely gathered in
a single place. I therefore aim for an exposition that is as self-contained as possible:
most statements are proved directly, and I rely on only a few external theorems.

\subsection{Linear and quadratic OT}

I begin by fixing the notations and recalling the fundamentals of discrete OT. 
The goal is to be concise, so readers seeking more details can refer to \citet{peyre2019computational}.

Standard linear OT aims to align two distributions according to a least-effort principle. 
I denote by $\simplex_n \triangleq \{\abf \in \R^n_+:\sum_{i=1}^n a_i = 1\}$ the probability simplex with $n$ bins.
Let $\Cbf \in \R^{n \times m}$ be a cost matrix, for instance encoding the pairwise distances between points from the two distributions, 
and let $\abf \in \simplex_n$ and $\bbf \in \simplex_m$ be probability vectors representing 
the available mass and the demand, respectively. 
The set of couplings, or transport plans, with prescribed marginals $\abf$ and $\bbf$, is defined by
\begin{equation}
\Pi(\abf, \bbf) \triangleq \{\Pbf \in \R_{+}^{n \times m}:  \Pbf \one_m = \abf, \Pbf^\top \one_n = \bbf\}\,,
\end{equation}
where $\one_n$ is the vector of ones. 

A special case of a coupling is when $n =m$ and when the mass is uniform $\abf = \bbf = \frac{1}{n} \one_n$: 
in this case a coupling $\Pbf$ can be supported by a permutation, that is $\Pbf \in \Perm(n)$ where
\begin{equation}
    \Perm(n) \triangleq \left\{\Pbf \in \R^{n \times n}: \exists \sigma \in \mathfrak{S}_n, 
    P_{ij} = \begin{cases} \frac{1}{n} \text{ if } j = \sigma(i) \\ 0 \text{ otherwise} \end{cases}\right\}\,,
\end{equation}
where $\mathfrak{S}_n$ is the set of all permutations of $\integ{n}$.

Linear OT searches for the transport plan $\Pbf \in \Pi(\abf,\bbf)$ that minimizes the ``shifting'' cost $\langle \Cbf, \Pbf \rangle \triangleq \sum_{ij} C_{ij} P_{ij}$. 
In the following, I note
\begin{equation}
    \label{eq:linear_ot}
    \tag{LinOT}
    \OT(\Cbf,\abf, \bbf)\triangleq \min_{\Pbf \in \Pi(\abf,\bbf)} \ \langle \Cbf, \Pbf \rangle\,.
\end{equation}
The quantity defined in problem \eqref{eq:linear_ot} is commonly referred to as the Wasserstein 
distance when $\Cbf$ represents a pairwise distance matrix. A key feature of this formulation is 
that the objective is \emph{linear} in $\Pbf$, in contrast with the ``quadratic’’ nature of the Gromov-Wasserstein 
problem. I introduce below a deliberately general version of this quadratic formulation, which will be specified in more detail later.

Let $\Lbf = (L_{ijkl})$ be a 4D tensor with $(i, j) \in \integ{n} \times \integ{m}, (k,l) \in \integ{n} \times \integ{m}$. 
The GW problem also aims to align the two distributions, but it does so by minimizing 
the quadratic cost $\sum_{ijkl} L_{ijkl} P_{ij} P_{kl}$. By introducing the tensor–matrix product $\Lbf \otimes \Pbf$, defined as the matrix
\[
\Lbf \otimes \Pbf \triangleq \big(\sum_{ij} L_{ijkl} P_{ij}\big)_{(k,l) \in \integ{n} \times \integ{m}},
\]
the objective minimized by GW can be written compactly as $\langle \Lbf \otimes \Pbf, \Pbf \rangle$.
I note 
\begin{equation}
    \label{eq:quad_ot}
    \tag{QuadOT}
    \GW(\Lbf,\abf, \bbf)\triangleq \min_{\Pbf \in \Pi(\abf,\bbf)} \ \langle \Lbf \otimes \Pbf, \Pbf \rangle\,.
\end{equation}
As announced, problem \eqref{eq:quad_ot} is \emph{quadratic} in $\Pbf$, which makes both 
the optimization and the theoretical analysis significantly more involved. 
In practice, the tensor $\Lbf$ is typically constructed as follows: given two ``intra’’ 
cost matrices $\Cbf \in \R^{n \times n}$ and $\overline{\Cbf} \in \R^{m \times m}$, 
which encode pairwise similarities within each space, 
together with a function $\Lcal: \R \times \R \to \R$ designed to measure how comparable two similarities are, 
one defines $\Lbf$ as
\begin{equation}
    \label{eq:Lbf}
    \Lbf = (L_{ijkl}) \text{ where } L_{ijkl}= \Lcal\left(C_{ik}, \overline{C}_{jl}\right)\,.
\end{equation}
A standard example is the squared-loss setting, 
where $\Lcal(a,b) = (a-b)^2$ and $\Cbf$ and $\overline{\Cbf}$ 
are the matrices of squared pairwise distances within each distribution. 
In what follows, I say that $\Lbf$ is symmetric if, for all $(i,j,k,l)$, one has $L_{ijkl} = L_{klij}$, 
meaning that swapping $i$ with $k$ and $j$ with $l$ leaves the tensor unchanged. 

We will also need the notion of the support of $\Pbf$, defined as the set of indices corresponding to the nonzero entries of the coupling:
\begin{equation}
    \label{eq:support}
    \supp(\Pbf)\triangleq \{(i,j) \in \integ{n} \times \integ{m}: P_{ij} > 0\}\,.
\end{equation}
Finally, two general definitions. For a convex set $\Ccal$, an \emph{extreme point} of $\Ccal$ 
is a point that cannot be written as a nontrivial convex 
combination\footnote{If $x$ is such point and $x = (1-t)y + tz$ with $0<t<1$ then $x=y=z$.} of other points in $\Ccal$.
In a graph $G = (V, E)$, a \emph{cycle} is a sequence of nodes $(u_1, u_2, \cdots, u_k, u_{k+1})$ in $V$, with $k \geq 3$, such that each consecutive pair $(u_i, u_{i+1})$ 
is connected by an edge in $E$, it starts and ends at the same vertex ($u_{k+1} =u_1$), and all other vertices are distinct.

\section{Some important properties of linear OT}

The fundamental properties of linear OT that I aim to investigate for GW 
in this note are the sparsity and monotonicity of optimal transport plans, as well as the ``tightness'' of the coupling relaxation. 
I detail these three properties below and provide proofs for each.

\subsection{Cyclical monotonicity} 

This is one of the most fundamental properties of linear OT, sometimes referred to as the \emph{shortening principle}. 
To illustrate, consider the following simple example: suppose that $(i,j)$ and $(i',j')$ are matched by $\Pbf$ that is optimal, they belong to $\supp(\Pbf)$. 
This means that the pairs $(i,j)$ and $(i',j')$ are matched because doing so incurs minimal cost. 
Intuitively, switching the matches to $(i,j')$ and $(i',j)$ should result in a higher cost; otherwise, $\Pbf$ would not be optimal. 

Formally, this can be seen by considering a matrix $\Qbf \in \R^{n \times m}$ that is identical to $\Pbf$ except at these four indices:
\begin{equation}
\begin{split}
Q_{ij} &= P_{ij} - \varepsilon,
\quad Q_{i'j'} = P_{i'j'} - \varepsilon, \\
Q_{ij'} &= P_{ij'} + \varepsilon,
\quad Q_{i'j} = P_{i' j} + \varepsilon, \\
\end{split}
\end{equation}
where $\varepsilon = \min\{P_{ij}, P_{i'j'}\} > 0$. It is then straightforward to verify that $\Qbf \in \Pi(\abf, \bbf)$, 
since the marginals remain unchanged and all entries are nonnegative by the choice of $\varepsilon$. Additionally,
\begin{equation}
    \langle \Cbf, \Qbf \rangle - \langle \Cbf, \Pbf \rangle = \varepsilon(-C_{ij} - C_{i'j'} + C_{ij'} + C_{i'j})\,.
\end{equation}
Using that $\Pbf$ is optimal implies $\langle \Cbf, \Qbf \rangle - \langle \Cbf, \Pbf \rangle \geq 0$ thus $C_{ij} + C_{i'j'} \leq  C_{ij'} + C_{i'j}$ 
which can be rephrased as\footnote{‘‘Lorsque le transport du deblai se fait de manière que la somme des produits des molécules par l’espace parcouru est un minimum, les routes de deux points quelconques A \& B, ne doivent plus se couper entre leurs extrémités, car la somme Ab + Ba des routes 
qui se coupent est toujours plus grande que la somme Aa + Bb de celles qui ne se coupent pas'' \citep{monge1781memoire}.} ``better not to cross the path'' !
This argument applies to just two pairs of points in the support, but the remarkable fact is that extending this property 
to all pairs leads to a full characterization: a transport plan is optimal if and only if, for every pair of points in its support, 
the total cost of the matched points is less than or equal to the total cost obtained by swapping them.
\begin{theorem}
For any costs $\Cbf$, a coupling $\Pbf \in \Pi(\abf, \bbf)$ is optimal for \eqref{eq:linear_ot} if and only if 
for any $N \in \mathbb{N}^{*}, (i_1, j_1),\cdots, (i_N, j_N) \in \supp(\Pbf)^N$ and permutation $\sigma \in \mathfrak{S}_N$,
\begin{equation}
\sum_{k=1}^{N} C_{i_k j_k} \leq \sum_{k=1}^{N} C_{i_k j_{\sigma(k)}}\,. 
\end{equation}
\end{theorem}

\noindent\begin{minipage}{0.55\linewidth}
The direction ``$\Pbf$ optimal $\implies$ monotonicity'' can be proved in the exact same way as the case $N = 2$ above. 
The other direction is a little bit more involved and I will not write the proof here (e.g., it can be proved using duality of linear OT).

\subsection{Sparsity of some optimal plans} 

Another key property is that, 
among all optimal transport plans, there exist \emph{sparse plans} with relatively 
few nonzero entries—specifically, no more than $n + m - 1$. 
To establish this, I first need a small result regarding the structure of coupling matrices.
Any coupling $\Pbf$ defines a bipartite graph $G(\Pbf) = (S \cup T, E)$ 
where $S = \integ{n}, T = \integ{m}$ are the source and target nodes that corresponds 
to the two distributions and $E = \supp(\Pbf)$ (see Figure \ref{fig:graphfig}).
\end{minipage}
\hfill
\begin{minipage}{0.4\linewidth}
    \centering
    \includegraphics[width=\linewidth]{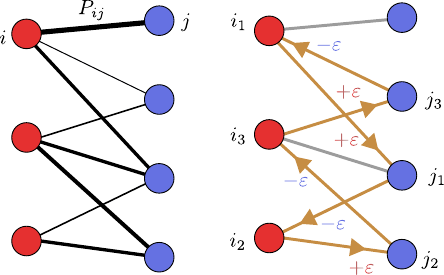}
    \captionof{figure}{\label{fig:graphfig} (Left) Bipartite graph $G(\Pbf)$ induced by $\Pbf$. 
    Weights on the edges are the values $P_{ij}$.
    (Right) It contains a 3-cycle $i_1, j_1, i_2, j_2, i_3, j_3, i_1$. 
    The forward edges $i \to j$ are marked with a $+\varepsilon$ perturbation, the backward with a $-\varepsilon$.}
\end{minipage}
\begin{proposition}
    \label{prop:nocycleareextreme}
 $\Pbf$ is an extreme point of $\Pi(\abf,\bbf)$ if and only if the graph $G(\Pbf)$ has no cycle.
\end{proposition}
\begin{proof}
    I first prove the direction ``$G$ has no cycle $\implies$ $\Pbf$ is an extreme point''. I prove it by contraposition. 
    Suppose that $\Pbf$ is not an extreme point: 
    there exists $\Pbf_1 \neq \Pbf_2 \in \Pi(\abf, \bbf)$ and $t\in (0,1)$ such that $\Pbf = (1-t) \Pbf_1 + t\Pbf_2$. 
    Taking $(i,j) \not\in \supp(\Pbf)$ implies that $ 0 = (1-t) [\Pbf_1]_{ij} + t[\Pbf_2]_{ij} \implies [\Pbf_1]_{ij} = [\Pbf_2]_{ij} = 0$. 
    Now consider $\Hbf = \Pbf_2 - \Pbf_1 \neq 0$, the previous reasoning implies that $\supp(\Hbf) \subseteq \supp(\Pbf)$. 
    
    Since $\Hbf \neq 0$ we can consider $(i_1, j_1)$ such that $H_{i_1 j_1} \neq 0$. 
    Looking at the line $i_1$ we have, since $\Pbf_1, \Pbf_2 \in \Pi(\abf, \bbf)$,  $\sum_j H_{i_1 j} = 0$ 
    thus there exists $j_2 \neq j_1$ such that $H_{i_1 j_2} \neq 0$. 
    We can do exactly the same for the column corresponding to $j_2$: we obtain an 
    $H_{i_2 j_2} \neq 0$ with $i_2 \neq i_1$. 
    We iterate this process and obtain a sequence $(i_1, j_1), (i_1, j_2), (i_2, j_2), \cdots$, 
    each in the support of $\Hbf$ and thus $\Pbf$.
    The size $N$ of this sequence is arbitrary, but since $\integ{n} \times \integ{m}$
    is finite, this process cannot continue indefinitely without revisiting an index. Consequently, the sequence must eventually contain a repeated row or column index, which yields a cycle in the support of $\Pbf$.

    I now prove the converse, and follow the same proof as \citet[Proposition 3.3]{peyre2019computational}. 
    Consider $\Pbf$ an extreme point. Suppose by contradiction that $G$ has a cycle. 
    Consider a 6-cycle $i_1, j_1, i_2,j_2, i_3, j_3, i_1$ as illustrated in Figure \ref{fig:graphfig} 
    (any other cycle with arbitrary length can be treated the same way).
    It corresponds to a set of edges $S=\{(i_1, j_1), (i_2, j_1), (i_2, j_2), \cdots, (i_1, j_3)\}$ in $\supp(\Pbf)$.
    As shown in this figure, on this cycle we mark the $i\to j$ as forward edges, and the $j \to i$ as backward edges.
    We consider a matrix $\Ebf$ defined as 
    \begin{equation}
        \label{eq:perturbation}
        E_{ij} = \begin{cases}
            &+ 1 \text{ if } (i,j) \text{ is a forward edge}\,, \\
            &- 1 \text{ if } (i,j) \text{ is a backward edge}\,, \\
            &0 \text{ otherwise }\,.
        \end{cases}
    \end{equation}
    Since this is a cycle, there are as many forward and backward edges, 
    and any node on this cycle receives exactly one $+1$ and one $-1$.
    Consequently, $\Ebf \one_m = 0 , \Ebf^\top \one_n = 0$.
    Now, for some sufficiently small $\varepsilon > 0$, define
    \[
    \Pbf_1 = \Pbf + \varepsilon \Ebf, \quad \Pbf_2 = \Pbf - \varepsilon \Ebf\,,
    \]
    so that $\Pbf = \frac{\Pbf_1 + \Pbf_2}{2}$. 
    Since the matrix $\Ebf$ has row and column sums equal to zero, 
    both $\Pbf_1$ and $\Pbf_2$ share the same marginals as $\Pbf$. 
    By choosing $0 < \varepsilon < \min_{(i,j) \in S} P_{ij}$, we ensure that $\Pbf_1$ and $\Pbf_2$ remain nonnegative and hence valid coupling matrices. 
    This shows that $\Pbf$ is not an extreme point, yielding a contradiction.

\end{proof}

This property of coupling matrices, together with the cyclical monotonicity discussed earlier, 
lead to the following result: some optimal plans in linear OT are both sparse and correspond to couplings that are extreme points of $\Pi(\abf, \bbf)$.

\begin{proposition}
    \label{prop:exists_sparse_coupling}
    For any cost $\Cbf$, there exists an optimal coupling $\Pbf \in \Pi(\abf, \bbf)$ for problem \eqref{eq:linear_ot} 
    that is an extreme point of $\Pi(\abf, \bbf)$. It satisfies $\card(\supp(\Pbf)) \leq n+m-1$. 
\end{proposition} 

\begin{proof}
Consider $\Pbf$ an optimal coupling with the \emph{smallest support}.
I will show that the corresponding graph has no cycle, and so it will be an extreme point by Proposition \ref{prop:nocycleareextreme}. 
I will conclude that $\card(\supp(\Pbf)) \leq n + m -1$.

Suppose that there is a cycle with length $k = 3$ in the support of $\Pbf$ as in Figure~\ref{fig:graphfig} 
$i_1 ,j_1 ,i_2 ,j_2 ,i_3 ,j_3 ,i_1$ (again, any longer cycle for $k \neq 3$ can be treated similarly).
We consider the perturbation $\Ebf$ as in the previous proof, by marking as forward the $i \to j$ edges and as backward the $j \to i$ edges,
with $\varepsilon = \min_{(i,j) \in B} P_{ij} > 0$ where $B$ is the set of \emph{backward edges} corresponding the cycle.
We define $\Qbf = \Pbf + \varepsilon\Ebf$.
With the same arguments as the previous proof, $\Qbf \in \Pi(\abf,\bbf)$ since $\Ebf \one_m = 0 , \Ebf^\top \one_n = 0$ 
and the fact that it is nonnegative (indeed for $(i,j)$ in forward edges $\varepsilon E_{ij} > 0$ and for $(i,j)$ in backward edges 
$P_{ij} + \varepsilon E_{ij} = P_{ij} - \varepsilon \geq 0$ since $\varepsilon$ is the smallest $P_{ij}$ among backward edges).

Moreover,
\begin{equation}
\langle \Cbf, \Qbf \rangle - \langle \Cbf, \Pbf \rangle = \sum_{ij} C_{ij} E_{ij}
= \varepsilon \Big(
C_{i_1, j_1} + C_{i_2, j_2} + C_{i_3, j_3}
- C_{i_2, j_1} - C_{i_3, j_2} - C_{i_1, j_3}
\Big).
\end{equation}
The RHS quantity is of the form $\sum_k C_{i_k, j_k} - \sum_k C_{i_{k+1}, j_k}$ with $(i_k, j_k)$ in the support. 
By cyclical monotonicity of the transport plan, this is $\leq 0$, hence $\Qbf$ is also an optimal coupling.
However, $\Qbf$ has strictly fewer strictly positive entries than $\Pbf$: the entries
$Q_{ij}$ where the minimum $\min_{(i,j) \in B} P_{ij}$ is attained become zero.
This is a contradiction since $\Pbf$ has the smallest support. 
Thus, the graph $G(\Pbf)$ has no cycle.

Finally, a bipartite graph with no cycle has at most $n+m-1$ edges.
Indeed, start with $n+m$ isolated vertices, so with a graph with $n+m$ components.  
Each new added edge either forms a cycle or connects two components. 
Since cycles are forbidden, each edge reduces the number of components by $1$. 
After $n+m-1$ edges there is a single component; adding another edge would create a cycle.

\end{proof}

This property lies at the heart of discrete algorithms for solving OT, 
such as the network simplex method. 
The key idea is to restrict attention to sparse transport plans—specifically, 
those whose support graphs contain no cycles—throughout the iterative optimization process. 
By focusing on such acyclic, sparse plans, these algorithms can efficiently navigate the feasible set while maintaining optimality 
(see discussions in \citealt[Chapter 3]{peyre2019computational}).

\subsection{Tightness of the coupling relaxation} 

The final important property I want to discuss concerns 
the special case of uniform weights, that is when $n = m$ and $\abf = \bbf = \frac{1}{n}\one_n$. 
In this setting, one can equivalently search for a permutation matrix instead of a general coupling, 
a formulation known as the Monge problem. 
A fundamental result, guaranteed by Birkhoff's theorem, is that these two formulations are equivalent, as I detail below.
\begin{theorem}[Birkhoff]
    \label{theo:Birkhoff}
    Extreme points of $\Pi(\frac{1}{n}\one_n, \frac{1}{n}\one_n)$ are the permutation matrices $\Perm(n)$.
\end{theorem}
\begin{proof}
First, if $\Pbf \in \Perm(n)$, then $\Pbf \in \Pi(\frac{1}{n}\one_n, \frac{1}{n}\one_n)$ 
and it is clear that the graph associated to $\supp(\Pbf)$ has no cycle (it is a permutation matrix, only one nonzero per line/column). 
Moreover, by Proposition \ref{prop:nocycleareextreme} we know that $\Pbf$ 
is an extreme point of $\Pi(\frac{1}{n}\one_n, \frac{1}{n}\one_n)$.

Conversely, we want to show that any extreme point of $\Pi(\frac{1}{n}\one_n, \frac{1}{n}\one_n)$ 
is a permutation matrix. 
The proof is a small adaptation of the proof of \citet[Theorem 2]{peyre2025optimal}. 

Consider $\Pbf \in \Pi(\frac{1}{n}\one_n, \frac{1}{n}\one_n)$ an extreme point. 
Suppose that it is not a permutation matrix. 
So there must be indices $(i_1, j_1), (i_1, j_2)$ with $j_1 \neq j_2$ in the support of $\Pbf$. 
Moreover, at this node $i_1$, we have $P_{i_1 j_1} < \frac{1}{n}$ otherwise $P_{i_1 j_2}$ would be zero (since in this case 
$P_{i_1 j_1} = \frac{1}{n}$ and all the mass would have been sent).
Thus, there must be an index $i_2 \neq i_1$ such that $P_{i_2 j_1} > 0$ (since $j_1$ does not receive enough mass). 
Similarly, there must be an index $i_3 \neq i_1, P_{i_3 j_2} > 0$. 

Now we have two pairs $(i_1, j_2), (i_3, j_2)$ with $i_1 \neq i_3$ in the support of $\Pbf$.
If $i_3 = i_2$ we have a cycle (make a drawing). If $i_3 \neq i_2$, then, from the same reasoning, 
$i_2$ must send mass to some $j_3$ and $i_3$ must send mass to some $j_4$: if $j_3 = j_4$ we have a cycle, otherwise we can iterate the process.
Since the graph has a finite number of vertices, 
there is a number of steps $N$ that necessarily leads to a cycle $i_1,j_1, \cdots, i_N, j_N, i_{N+1} = i_1$. 

This cycle can be used to split the graph into two set of edges and construct $\Pbf_1 \neq \Pbf_2 \in \Pi(\frac{1}{n}\one_n, \frac{1}{n}\one_n)$ 
such that $\Pbf = \frac{1}{2}(\Pbf_1 + \Pbf_2)$, contradicting the hypothesis that $\Pbf$ is an extreme point. 
These matrices can be obtained exactly as in the proof of Proposition \ref{prop:nocycleareextreme}: 
we mark forward and backward edges with $+1$ and $-1$, and we define $\Ebf$ as in \eqref{eq:perturbation} 
with $\varepsilon$ sufficiently small.
\end{proof}

Combining Proposition \ref{prop:exists_sparse_coupling} with this theorem yields the well-known result often summarized as ``Monge = Kantorovich'' result:
\begin{corollary}
 Let $n=m, \ \abf=\bbf =\frac{1}{n}\one_n$. There exists an optimal solution of \eqref{eq:linear_ot} that solves 
     $\min_{\Pbf \in \Perm(n)} \langle \Cbf, \Pbf \rangle$ and this quantity is equal to $\OT(\Cbf, \abf,\bbf)$.   
\end{corollary}

\begin{proof}
    First, as any permutation is a valid coupling $\OT(\Cbf, \abf,\bbf) \leq \min_{\Pbf \in \Perm(n)} \langle \Cbf, \Pbf \rangle$. 
    Proposition \ref{prop:exists_sparse_coupling} shows that there exists an optimal solution of \eqref{eq:linear_ot} that is an extreme point, 
    which is a permutation by Birkhoff's theorem. 
\end{proof}

\section{What about GW optimal transport plans ?}

The natural question now is: do these properties extend to the GW problem \eqref{eq:quad_ot}? 
A spoiler: in general, it is much harder to establish such properties for GW, so the answer is usually no. 
Nevertheless, I will describe one sufficient condition, commonly used in the literature, that 
allows similar results to be derived for the GW case.

\subsection{Conditionally negative semi-definite tensor}

This property stems from the observation that the concavity of the GW loss 
can be exploited to derive results about the extremality of its solutions. 
It was first formally introduced for GW in \citet{sejourne2021unbalanced} 
and has since been applied in works such as 
\citet{beier2023multi, memoli2024comparison, dumont2025existence, assel2025distributional, houry2026gromov}. 
It corresponds to a particular structure on the 4D tensor $\Lbf$. 
The formal definition is given below, and Section \ref{sec:sec_cnd} will discuss in detail the conditions under which this property holds.

\begin{definition}
A symmetric 4D tensor $\Lbf$ is conditionally negative semi-definite (CND) with respect to 
$\overline{\Pi} \triangleq \Pi(\abf, \bbf)-\Pi(\abf, \bbf)=\{\Pbf_1-\Pbf_2, \ (\Pbf_1, \Pbf_2) \in \Pi(\abf, \bbf) \times \Pi(\abf, \bbf)\}$ if
\begin{equation}
\forall \Qbf \in \overline{\Pi}, \  \langle \Lbf \otimes \Qbf, \Qbf\rangle \leq 0\,.
\end{equation}
\end{definition}
As suggested above, the lemma below shows that it is exactly 
a reformulation of the fact that the GW loss is concave.
\begin{lemma}
    \label{lemma:proof_cnd}
    The 4D tensor $\Lbf$ is CND with respect to $\overline{\Pi}$ if and only if $f: \Pbf \in \Pi(\abf, \bbf) \to \langle \Lbf \otimes \Pbf, \Pbf\rangle$ is concave, 
    that is, the GW loss function is concave on $\Pi(\abf, \bbf)$.
\end{lemma} 
\begin{proof}
    The function $f$ is continuous, so it is concave if and only if it satisfies the midpoint inequality
    $f(\frac{\Pbf_1 + \Pbf_2}{2}) \geq \frac{1}{2}(f(\Pbf_1) + f(\Pbf_2))$ for any $\Pbf_1, \Pbf_2 \in \Pi(\abf,\bbf)$.
    However, since $\Lbf$ is symmetric,
    \begin{equation*}
        \begin{split}
        f(\frac{\Pbf_1 + \Pbf_2}{2}) - \frac{1}{2}(f(\Pbf_1) + f(\Pbf_2)) &= \frac{1}{4}  \langle \Lbf \otimes \Pbf_1, \Pbf_1\rangle 
        + \frac{1}{4} \langle \Lbf \otimes \Pbf_2, \Pbf_2\rangle + \frac{2}{4} \langle \Lbf \otimes \Pbf_1, \Pbf_2\rangle \\
        &- \frac{2}{4} \langle \Lbf \otimes \Pbf_1, \Pbf_1\rangle- \frac{2}{4} \langle \Lbf \otimes \Pbf_2, \Pbf_2\rangle \\
        &=\frac{1}{4}(2\langle \Lbf \otimes \Pbf_1, \Pbf_2\rangle - \langle \Lbf \otimes \Pbf_1, \Pbf_1\rangle - \langle \Lbf \otimes \Pbf_2, \Pbf_2\rangle) \\
        &= -\frac{1}{4} \langle \Lbf \otimes \Qbf, \Qbf\rangle\,.
        \end{split}
    \end{equation*}
\end{proof}

Before stating \emph{when} this property holds, I first describe \emph{what} consequences it has for the GW problem.

\subsection{First consequence: sparsity of some optimal plans} 

The key idea is that minimizing a concave function over a bounded convex polytope 
can be achieved by considering \emph{only the extreme points of the polytope}. 
By combining this with the fact that the extreme points of the set of coupling matrices are sparse, 
one can deduce the sparsity of some GW solutions. 

More precisely, let $\Ccal \subset \R^d$ be a convex set that can be expressed as the convex hull of its extreme points, 
and let $f: \Ccal \to \R$ be a continuous concave function. 
Then there exists an extreme point of $C$ that solves\footnote{This extends to any compact convex set and is known as Bauer's minimum principle.}
\begin{equation*}
\min_{\xbf \in \Ccal} \ f(\xbf) \,.
\end{equation*}
Indeed, let $\xbf \in \Ccal$ be a minimizer of $f$. Since $\xbf$ lies in the convex hull of extreme points of $\Ccal$, by Carathéodory's theorem it can be expressed as a convex combination of at most $d+1$ extreme points:
$\xbf = \sum_{i=1}^{d+1} \lambda_i \xbf_i, \lambda_i \geq 0, \sum_{i=1}^{d+1} \lambda_i = 1$.
By concavity and Jensen's inequality,
\begin{equation*}
f(\xbf) = f\Big(\sum_{i=1}^{d+1} \lambda_i \xbf_i\Big) \ge \sum_{i=1}^{d+1} \lambda_i f(\xbf_i) \ge \min_{i} f(\xbf_i)\,.    
\end{equation*}
Thus, there exists at least one index $i$ such that $f(\xbf_i) = f(\xbf)$, meaning that the corresponding $\xbf_i$, 
an extreme point of $\Ccal$, is a minimizer of $f$. 
In particular, this reasoning applies whenever every point in $\Ccal$ can be expressed as a convex combination of its extreme points.
The good news is: $\Ccal = \Pi(\abf,\bbf)$ is such a set! 
\begin{proposition}
    \label{prop:coupling_convex}
    Any point $\Pbf \in \Pi(\abf, \bbf)$ can be written as $\Pbf =  \sum_{i=1}^{K} \lambda_i \Pbf_i$ where $K \geq 1, \Pbf_1, \cdots, \Pbf_K$ 
    are extreme points of $\Pi(\abf, \bbf)$ and $\lambda_i \geq 0, \sum_{i=1}^K \lambda_i = 1$.
\end{proposition}
\begin{proof}
To prove this result, one could appeal to general theorems about bounded convex polytopes, but here I provide a constructive proof. 
If $\Pbf$ is already an extreme point, the statement is immediate. 
Otherwise, suppose $\Pbf$ is not an extreme point; the proof then proceeds in a manner very similar to the previous arguments.
From Proposition \ref{prop:nocycleareextreme}, then the graph $G(\Pbf)$ contains a cycle.
Consider the 3-cycle $i_1, j_1, i_2,j_2, i_3, j_3, i_1$ in Figure \ref{fig:graphfig} (any longer cycle leads to the same idea).
We mark again the forward and backward edges as in the figure and consider 
$\varepsilon^- = \min_{(i,j) \in B} P_{ij} > 0$ and $\varepsilon^+ = \min_{(i,j) \in F} P_{ij} > 0$ where $B, F$ are the sets of backward and forward 
edges and $\Ebf$ as in \eqref{eq:perturbation}.
Now we define
\begin{equation}
    \Pbf_1 = (\Pbf + \varepsilon^- \Ebf), \ \Pbf_2 = (\Pbf - \varepsilon^+ \Ebf), \ \lambda = \frac{\varepsilon^+}{\varepsilon^+ + \varepsilon^-}\,,
\end{equation}
such that $1- \lambda = \frac{\varepsilon^-}{\varepsilon^+ + \varepsilon^-}$. 
With similar reasoning as before we can check that $\Pbf_1, \Pbf_2$ have the same marginals as $\Pbf$ and are both nonnegative. 
Also, $\Pbf = \lambda \Pbf_1 + (1-\lambda) \Pbf_2$.
The crucial point is that we have removed at least one edge in each $\Pbf_1$ and $\Pbf_2$; that is 
$\card(\supp(\Pbf_1)), \card(\supp(\Pbf_2)) < \card(\supp(\Pbf))$.
In $\Pbf_1$ we removed the backward edges corresponding to $\min_{(i,j) \in B} P_{ij}$ and in $\Pbf_2$ the forward edges corresponding to $\min_{(i,j) \in F} P_{ij}$.
If $\Pbf_1$ and $\Pbf_2$ do not have a cycle we are done. 
Otherwise, we can iterate the process on $\Pbf_1, \Pbf_2$ until there is no cycle anymore. In the end we end up with $\Pbf = \sum_i \lambda_i \Pbf_i$ 
with all the $\Pbf_i$ that have no cycle, thus are extreme by Proposition \ref{prop:nocycleareextreme}. 
\end{proof}

Using this result, together with the earlier reasoning on concave functions, we can conclude that some GW optimal plans are sparse.
\begin{corollary}
    \label{corr:firstcorrgromov}
    When the 4D tensor $\Lbf$ is CND with respect to $\overline{\Pi}$, 
    there exists an optimal solution $\Pbf$ of problem \eqref{eq:quad_ot} which is an extreme point of $\Pi(\abf,\bbf)$ 
    and with $\card(\supp(\Pbf)) \leq n+m -1$.
\end{corollary} 
\begin{proof}
    When $\Lbf$ is CND the GW loss is concave and continuous on $\Pi(\abf,\bbf)$. As any coupling can be written as convex combination of extreme points, 
    as detailed in Proposition \ref{prop:coupling_convex}, so there exists an extreme point that is an optimal solution by the previous discussion.
    But as written in the proof of Proposition \ref{prop:exists_sparse_coupling}, 
    since the bipartite graph associated to $\supp(\Pbf)$ has no cycle, $\card(\supp(\Pbf)) \leq n+m -1$.
\end{proof}

\subsection{Second consequence: tightness of the coupling relaxation}

Similarly, when the tensor $\Lbf$ is CND, one can show that the coupling relaxation is tight—that is, 
a ``Monge = Kantorovich''–type result holds for GW. 
This observation was already noted for quadratic programs in the great paper \citet{maron2018probably}. 
By combining the facts that the extreme points of $\Pi\big(\frac{1}{n} \one_n, \frac{1}{n} \one_n\big)$ 
are permutation matrices (Theorem \ref{theo:Birkhoff}) and that at least one extreme point is an optimal solution 
(Corollary \ref{corr:firstcorrgromov}), we obtain:
\begin{corollary}
 Let $n=m, \ \abf=\bbf =\frac{1}{n}\one_n$. Suppose that the 4D tensor $\Lbf$ is CND with respect to $\overline{\Pi}$. 
 There exists an optimal solution of \eqref{eq:quad_ot} that solves 
    $\min_{\Pbf \in \Perm(n)} \langle \Lbf \otimes \Pbf, \Pbf \rangle$ and this quantity is equal to $\GW(\Lbf, \abf,\bbf)$.   
\end{corollary}

\subsection{Third consequence: a small detour around the bilinear relaxation}

Another noteworthy consequence of the CND case is that a certain \emph{bilinear relaxation} becomes exact. 
Before wrapping up, I briefly introduce this concept. The bilinear problem, first formally introduced for OT in \citet{titouan2020co}, is formulated as
\begin{equation}
    \label{eq:bilin_ot}
    \tag{BilinOT}
    \min_{\Pbf_1, \Pbf_2 \in \Pi(\abf,\bbf)} \ \langle \Lbf \otimes \Pbf_1, \Pbf_2 \rangle\,.
\end{equation}
In other words, instead of seeking a single global transport plan, I look for two plans that realign the distributions. 
From a numerical standpoint, this can be advantageous because the problem becomes bilinear rather than quadratic, 
which opens the door to algorithms based on linear OT \citep{titouan2020co, sejourne2021unbalanced, beier2023multi}. 
A simple bound shows that this formulation is indeed a relaxation: 
\[
\min_{\Pbf_1, \Pbf_2 \in \Pi(\abf,\bbf)} \langle \Lbf \otimes \Pbf_1, \Pbf_2 \rangle \leq \GW(\Lbf, \abf, \bbf),
\] 
and the natural question is whether this relaxation is tight. In the CND case, the answer is affirmative.
\begin{proposition}
    \label{prop:bilinear_relaxation}
    If the tensor $\Lbf$ is CND with respect to $\overline{\Pi}$ then \eqref{eq:bilin_ot} and \eqref{eq:quad_ot} are equivalent. More precisely, 
    if $(\Pbf_1, \Pbf_2)$ is optimal for \eqref{eq:bilin_ot} then both $\Pbf_1$ and $\Pbf_2$ are optimal solutions for \eqref{eq:quad_ot} 
    and if $\Pbf$ is optimal for \eqref{eq:quad_ot} then $(\Pbf, \Pbf)$ is an optimal solution for \eqref{eq:bilin_ot}.
    In this case, $\GW(\Lbf, \abf, \bbf) = \min_{\Pbf_1, \Pbf_2 \in \Pi(\abf,\bbf)} \ \langle \Lbf \otimes \Pbf_1, \Pbf_2 \rangle$.
\end{proposition}
\begin{proof}
    In the proof I define $g(\Pbf_1, \Pbf_2) \triangleq \langle \Lbf \otimes \Pbf_1, \Pbf_2 \rangle$ the bilinear loss 
    and $f(\Pbf) \triangleq g(\Pbf, \Pbf)$ the GW loss, which is concave due to the hypothesis (Lemma \ref{lemma:proof_cnd}).
    Moreover, a small calculus shows that
    \begin{equation}
        \label{eq:the_eq_for_bilin}
        g(\Pbf_1, \Pbf_2) = \frac{1}{2}(f(\Pbf_1 + \Pbf_2) - f(\Pbf_1) - f(\Pbf_2))\,.
    \end{equation} 
    Since $f$ is concave it satisfies the midpoint inequality $f(\frac{\Pbf_1 + \Pbf_2}{2})\geq \frac{1}{2}(f(\Pbf_1)+ f(\Pbf_2))$ which gives $f(\Pbf_1 + \Pbf_2) \geq 2(f(\Pbf_1)+ f(\Pbf_2))$ (since $f( 2\Pbf) = 4f(\Pbf)$).
    Combining with \eqref{eq:the_eq_for_bilin} we get $g(\Pbf_1, \Pbf_2) \geq \frac{1}{2}(f(\Pbf_1)+ f(\Pbf_2)) \geq \min\{f(\Pbf_1), f(\Pbf_2)\} 
    \geq \min_\Pbf f(\Pbf) = \min_\Pbf g(\Pbf, \Pbf) = \GW(\Lbf, \abf, \bbf)$ and thus $\min_{\Pbf_1, \Pbf_2 \in \Pi(\abf,\bbf)} \ \langle \Lbf \otimes \Pbf_1, \Pbf_2 \rangle \geq \GW(\Lbf, \abf, \bbf)$.
    Using the converse inequality shows $\GW(\Lbf, \abf, \bbf) = \min_{\Pbf_1, \Pbf_2 \in \Pi(\abf,\bbf)} \ \langle \Lbf \otimes \Pbf_1, \Pbf_2 \rangle$ and the fact that the solutions are equivalent.
\end{proof}

\subsection{When is the tensor CND ? The case of separable losses \label{sec:sec_cnd}}

Now that I have presented some consequences of the CND case, I will explain when this situation actually occurs.
As written in the introduction, in most of the applications the tensor can be written as $L_{ijkl}= \Lcal\left(C_{ik}, \overline{C}_{jl}\right)$ 
for some loss function $\Lcal : \R \times \R \to\R$.

In fact, a lot of losses $\Lcal$ for GW that are used in practice are \emph{separable}, mainly for practical reasons: 
as described in \citet{peyre2016gromov} this reduces the computation complexity of the GW loss from $\mathcal{O}(n^2 m ^2)$ to $\mathcal{O}(n m^2 + m n^2)$. 
These losses can be written as
\begin{equation}
    \label{eq:separable_loss}
    \Lcal(a, b) = f_1(a) + f_2(b) - h_1(a)h_2(b)\,,
\end{equation}
and they cover a wide range of loss functions.
For instance, they include all \emph{Bregman divergences} that can be written as 
$\Lcal(a, b) = \phi(a) - \phi(b) - \phi'(b)(a-b) \geq 0$ 
for some (strictly) convex and differentiable function $\phi$.
This corresponds to $f_1(a) = \phi(a), \ f_2(b) = - \phi(b) + \phi'(b)b, \ h_1(a) = a, \ h_2(b) = \phi'(b)$.
Notable examples include the squared loss $$\Lcal(a,b) = \Lcal_2(a,b) \triangleq (a-b)^2\,,$$ and the Kullback-Leibler divergence 
$$\Lcal(a,b) =  \Lcal_{\KL}(a,b) \triangleq a \log(a/b) - a + b\,,$$ which corresponds to the Bregman divergence associated to $\phi(x) = x\log(x) - x$.
When the loss is separable the expression for the GW loss simplifies to 
\begin{equation}
    \label{eq:simplification}
    \langle \Lbf \otimes \Pbf, \Pbf\rangle = \langle f_1(\Cbf) \abf \one_m^\top + \one_n \bbf^\top f_{2}(\overline{\Cbf})^\top, \Pbf\rangle - \langle h_{1}(\Cbf) \Pbf h_2(\overline{\Cbf})^\top, \Pbf \rangle\,,
\end{equation}
as shown in \citealt[Proposition 1]{peyre2016gromov}. In \eqref{eq:simplification}, 
the expressions $f_1(\Cbf)$, $f_2(\overline{\Cbf})$, $h_1(\Cbf)$, and $h_2(\overline{\Cbf})$ are to be interpreted component-wise.
The goal of this section is to characterize the CND property for these separable losses. I will use the following definition:
\begin{definition}
    A symmetric matrix $\Cbf \in \R^{n \times n}$ is called conditionally negative semi-definite (\textit{resp.} positive semi-definite), 
    abbreviated as CND (\textit{resp.} CPD), if for any $\ubf \in \R^n \text{ s.t. } \ubf^\top \one_n = 0$ 
    we have $\ubf^\top \Cbf \ubf \leq 0$ (\textit{resp.} $\geq 0$).
\end{definition}
I will also use the following simple result: a matrix is CND if and only if it is negative semi-definite after centering its rows and columns.
\begin{lemma}
    \label{lemma:centering}
    Let $\Hbf_n \triangleq \Ibf_n - \frac{1}{n} \one_n \one_n^\top$ be the centering matrix where $\Ibf_n$ is the $n \times n$ identity matrix. 
    $\Cbf \in \R^{n \times n}$ is CND (\textit{resp.} CPD) if and only if $\Hbf_n \Cbf \Hbf_n$ is negative semi-definite (\textit{resp.} positive semi-definite).
\end{lemma}
    The proof is straightforward by using that, for any  $\ubf \in \R^n, (\Hbf \ubf)^\top \one_n = 0$.
The following main result characterizes the CND property for the class of separable losses.
\begin{proposition}
    \label{prop:separable_loss_cnd}
    Let $\Lbf$ be a 4D tensor that can be written as $L_{ijkl}= \Lcal\left(C_{ik}, \overline{C}_{jl}\right)$ 
    for a separable loss $\Lcal(a,b) = f_1(a) + f_2(b) - h_1(a)h_2(b)$ and symmetric matrices $\Cbf, \overline{\Cbf}$. Take $\abf \in \Delta_n, \bbf \in \Delta_m$ with positive entries.
The following statements are equivalent:
\begin{enumerate}[label=(\roman*)]
\item The 4D tensor $\Lbf$ is CND with respect to $\overline{\Pi}$.
\item The GW loss $\Pbf\mapsto
\langle \Lbf\otimes\Pbf,\Pbf\rangle$
is concave on $\Pi(\abf,\bbf)$.
\item One of the following conditions holds:
\begin{enumerate}[label=(\alph*)]
\item $h_1(\Cbf)$ and $h_2(\overline{\Cbf})$ are both CND,
\item $h_1(\Cbf)$ and $h_2(\overline{\Cbf})$ are both CPD,
\item $\Hbf_n h_1(\Cbf)\Hbf_n=0$,
\item $\Hbf_m h_2(\overline{\Cbf})\Hbf_m=0$.
\end{enumerate}
\end{enumerate}
\end{proposition}
\begin{proof}
    The equivalence between the two first points was 
    already established in Lemma \ref{lemma:proof_cnd}; here, we only prove the equivalence between the last two points.
    I define $\Mbf \triangleq f_1(\Cbf) \abf \one_m^\top + \one_n \bbf^\top f_{2}(\overline{\Cbf})^\top$. 
    To ease the notation I note $\Cbf_1 \triangleq h_{1}(\Cbf), \Cbf_2 \triangleq -h_2(\overline{\Cbf})$.
    As written in \eqref{eq:simplification} the loss can be written 
    as 
    $f(\Pbf) \triangleq \langle \Lbf \otimes \Pbf, \Pbf\rangle = \langle \Mbf, \Pbf\rangle + \langle \Cbf_1 \Pbf \Cbf_2, \Pbf \rangle 
    =  \langle \Mbf, \Pbf\rangle + \tr(\Pbf^\top \Cbf_1 \Pbf \Cbf_2)$ (the matrices $\Cbf_1, \Cbf_2$ are symmetric we can remove the transpose).
    
    Suppose that $h_1(\Cbf), h_2(\overline{\Cbf})$ are both CND, we show that $f$ is concave on $\Pi(\abf, \bbf)$.
    To do this we first show that, for any $\Pbf_1, \Pbf_2 \in \Pi(\abf, \bbf)$,
    \begin{equation}
        \tr\left(\Qbf^\top \Cbf_1 \Qbf \Cbf_2\right) \leq 0 \text{ where } \Qbf \triangleq \Pbf_1 - \Pbf_2 \in \R^{n \times m}\,.
    \end{equation}
    This matrix satisfies $\Qbf \one_m = 0, \Qbf^\top \one_n = 0$ since the couplings have the same marginals.
    Also, with the centering matrices $\Hbf_n , \Hbf_m$ defined in Lemma \ref{lemma:centering}, we have
    \begin{equation}
        \Hbf_n \Qbf \Hbf_m = (\Ibf_n - \frac{1}{n}\one_n \one_n^\top) \Qbf \Hbf_m = (\Qbf - \frac{1}{n}\one_n (\one_n^\top \Qbf)) \Hbf_m = \Qbf \Hbf_m = \Qbf\,.
    \end{equation}
    Hence, 
    \begin{equation}
        \label{eq:tosquareroot}
        \begin{split}
            \tr\left(\Qbf^\top \Cbf_1 \Qbf \Cbf_2\right) &= \tr\left((\Hbf_n \Qbf \Hbf_m)^\top \Cbf_1 (\Hbf_n \Qbf \Hbf_m) \Cbf_2\right) 
            = \tr\left(\Qbf^\top (\Hbf_n \Cbf_1 \Hbf_n) \Qbf (\Hbf_m \Cbf_2 \Hbf_m)\right) \\
            &= -\tr\left(\Qbf^\top [-(\Hbf_n \Cbf_1 \Hbf_n)] \Qbf (\Hbf_m \Cbf_2 \Hbf_m)\right)\,. \\
        \end{split}
    \end{equation}
    Since $h_1(\Cbf), h_2(\overline{\Cbf})$ are both CND, 
    $\Cbf_1 = h_1(\Cbf)$ is CND and $\Cbf_2 = - h_2(\overline{\Cbf})$ is CPD.
    Thus $\Abf \triangleq -(\Hbf_n \Cbf_1 \Hbf_n), \Bbf \triangleq (\Hbf_m \Cbf_2 \Hbf_m)$ 
    are symmetric positive semi-definite by Lemma \ref{lemma:centering}, 
    thus they admit a square root.
    Consequently, \eqref{eq:tosquareroot} implies 
    \begin{equation}
        \label{eq:topositive}
        \begin{split}
            \tr\left(\Qbf^\top \Cbf_1 \Qbf \Cbf_2\right) &= -\tr\left(\Qbf^\top \Abf^{1/2} \Abf^{1/2} \Qbf \Bbf^{1/2} \Bbf^{1/2}\right) 
            = -\|\Abf^{1/2} \Qbf \Bbf^{1/2}\|_F^2 \leq 0\,,
        \end{split}
    \end{equation}
    where $\|\cdot\|_F$ is the Frobenius norm. 
    The concavity of $f$ on $\Pi(\abf, \bbf)$ is a direct consequence, since it shows the midpoint inequality 
    $f(\frac{\Pbf_1 + \Pbf_2}{2}) \geq \frac{1}{2}(f(\Pbf_1) + f(\Pbf_2))$ for any $\Pbf_1, \Pbf_2 \in \Pi(\abf,\bbf)$.
    Indeed, with the same calculus as in the proof of Lemma \ref{lemma:proof_cnd},
    \begin{equation}
        \label{eq:midpoint_specific}
        \begin{split}
        f(\frac{\Pbf_1 + \Pbf_2}{2}) - \frac{1}{2}(f(\Pbf_1) + f(\Pbf_2)) &\stackrel{*}{=}  -\frac{1}{4} \tr(\Qbf^\top \Cbf_1 \Qbf \Cbf_2) 
        = \frac{1}{4}\|\Abf^{1/2} \Qbf \Bbf^{1/2}\|_F^2 \geq  0\,,
        \end{split}
    \end{equation}
    where in $(*)$ the linear terms get cancelled. 
    
    When $\Cbf_1, \Cbf_2$ are both CPD we make the same reasoning, but instead we consider 
    $\Abf = (\Hbf_n \Cbf_1 \Hbf_n), \Bbf = -(\Hbf_m \Cbf_2 \Hbf_m)$: this does not change the conclusion. When $\Hbf_n \Cbf_1 \Hbf_n = 0$ or $\Hbf_m \Cbf_2 \Hbf_m = 0$ we have $\Abf = 0$ or $\Bbf = 0$ in \eqref{eq:topositive} which implies in both cases $\tr(\Qbf^\top \Cbf_1 \Qbf \Cbf_2) = 0$ which implies the midpoint equality in \eqref{eq:midpoint_specific}, thus concavity by continuity of $f$. 
    Consequently any $(a), (b), (c), (d)$ of $(iii)$ implies $(ii)$.
    
    To prove $(iii) \implies (ii)$ I prove now $\text{no} (iii) \implies \text{no} (ii)$. We suppose simultaneously that 
    \begin{itemize}
        \item $h_1(\Cbf)$ and $h_2(\overline{\Cbf})$ are not both CND.
        \item $h_1(\Cbf)$ and $h_2(\overline{\Cbf})$ are not both CPD.
        \item $\Hbf_n \Cbf_1 \Hbf_n\neq 0$.
        \item $\Hbf_m \Cbf_2 \Hbf_m \neq 0$,
    \end{itemize}
    and show that $f$ is not concave. 
    To do this we will construct $\Pbf_1, \Pbf_2 \in \Pi(\abf, \bbf)$ such that the midpoint difference in the LHS of \eqref{eq:midpoint_specific} is negative.  We will construct $\Qbf \in \R^{n \times m}$ with $\Qbf \one_m = 0, \Qbf^\top \one_n = 0$ and such that $\tr(\Qbf^\top \Cbf_1 \Qbf \Cbf_2) > 0$.

    First, ``$h_1(\Cbf)$ and $h_2(\overline{\Cbf})$ are not both CND'' and ``$h_1(\Cbf)$ and $h_2(\overline{\Cbf})$ are not both CPD'' together with $\Hbf_n \Cbf_1 \Hbf_n\neq 0$ and $\Hbf_m \Cbf_2 \Hbf_m \neq 0$ implies that there is a nonzero eigenvalue $\mu$ of $\Hbf_n \Cbf_1 \Hbf_n$ and a nonzero eigenvalue $\nu$ of  $\Hbf_m \Cbf_2 \Hbf_m$ with the same signs, i.e., $\mu \times \nu > 0$.
    Let $\ubf$ and $\vbf$ be the eigenvectors of $\mu$ and $\nu$. Setting $\Qbf\stackrel{\Delta}{=}\Hbf_n\ubf\vbf^\top\Hbf_m$ we obtain $ \tr(\Qbf^\top\Cbf_1\Qbf\Cbf_2) =\mu \times \nu > 0$ and this is true for any $\Qbf' = \alpha \Qbf$ with $\alpha > 0$. Also, $\Qbf \one_m = 0, \Qbf^\top \one_n = 0$.

    We show that we can decompose $\Qbf$ as $\Qbf = \Pbf_1 - \Pbf_2$ for some  $\Pbf_1, \Pbf_2 \in \Pi(\abf, \bbf)$.
    For this we consider $\Pbf_1 = \abf \bbf^\top + \varepsilon \Qbf, \Pbf_2 = \abf \bbf^\top - \varepsilon \Qbf$ for 
    $0 < \varepsilon < \min_{(i,j): Q_{ij} \neq 0}\frac{a_i b_j}{|Q_{ij}|}$ small enough.
    We have $\Pbf_1 - \Pbf_2 =  2 \varepsilon \Qbf$ and $\Pbf_1, \Pbf_2 \in \Pi(\abf, \bbf)$ since $\Qbf$ has zero row and column sums. 
    We consider $\Qbf' = 2 \varepsilon \Qbf$: by the previous reasoning  $\tr(\Qbf'^\top \Cbf_1 \Qbf' \Cbf_2)  > 0$ which implies that $f(\frac{\Pbf_1 + \Pbf_2}{2}) < \frac{1}{2}(f(\Pbf_1) + f(\Pbf_2))$ by \eqref{eq:midpoint_specific} so that $f$ is not concave.
\end{proof}

Below, I present a few examples that satisfy the conditions of Proposition \ref{prop:separable_loss_cnd}. 
They can all be viewed as corollaries of the Bregman divergence setting 
associated with a convex function $\phi$. 
By the previous proposition, the corresponding GW loss is concave if and only if both $\Cbf$ and $\phi'(\overline{\Cbf})$ are CND or CPD.

\paragraph{Example 1: Squared case $\Lcal = \Lcal_2$.}  
The squared case corresponds simply to $h_1(a) = 2a, h_2(b) = b$, so the problem is concave whenever both $\Cbf$ and $\overline{\Cbf}$ are CND (or CPD). 
Examples of CND and CPD matrices can be found in the comprehensive treatment by \citet{wendland2004scattered} 
or in \citet[Section 2]{maron2018probably}. 
The goal here is not to provide an exhaustive list of examples. 
However, a particularly important and widely used setting in GW is when both $\Cbf$ and $\overline{\Cbf}$ are squared Euclidean distance matrices, i.e.,  
\[
C_{ik} = \|\xbf_i - \xbf_k\|_2^2, \quad \overline{C}_{jl} = \|\ybf_j - \ybf_l\|_2^2 \text{ for some } \xbf_1, \cdots, \xbf_n \text{ and } \ybf_1, \cdots, \ybf_m\,.
\]  
It is easy to see that if $\ubf^\top \one_n = 0$, then
\[
\ubf^\top \Cbf \ubf = -2\Big\|\sum_i u_i \xbf_i \Big\|_2^2 \leq 0,
\] 
so both matrices are CND. In this case, the GW problem is concave, admits sparse optimal solutions, and both the coupling and bilinear relaxations are tight.
Remarkably, this example essentially captures the whole picture thanks to the celebrated Schoenberg theorem \citep{schoenberg1938metric}: in short, if $\Cbf$ is a symmetric $n \times n$ matrix with zero diagonal, then $\Cbf$ is CND if and only if it can be written as
$
C_{ik} = \|\xbf_i - \xbf_k\|_\Hcal^2,
$
for some points $\xbf_1, \dots, \xbf_n$ in a Hilbert space $\Hcal$. In other words, in the squared case $\Lcal = \Lcal_2$, most CND situations correspond precisely to squared distance matrices for $\Cbf$ and $\overline{\Cbf}$.

\paragraph{Example 2: Kullback-Leibler case $\Lcal = \Lcal_{\operatorname{KL}}$.}  
The KL case is interesting because it highlights the role of a particular type of matrices. 
It corresponds to a Bregman divergence with $\phi(x) = x \log(x) - x$, i.e., $\phi'(x) = \log(x)$. 
The problem is concave when both $\Cbf$ and $\log(\overline{\Cbf})$ are CND or CPD. 
Matrices (with positive entries) whose logarithm is CPD are well-studied in the literature: they are called infinitely divisible matrices. One characterization is that any elementwise 
power of the matrix should be positive semi-definite \citep[Section 3.2]{bhatia2006infinitely}.

\begin{remark}
The previous conclusions remain valid if a linear term is added to the loss, i.e., 
for objectives of the form $\Pbf \mapsto \langle \Abf, \Pbf \rangle + \langle \Lbf \otimes \Pbf, \Pbf \rangle$. 
To study concavity, it suffices to analyze the quadratic part. 
As long as $\Lbf$ is CND with respect to $\overline{\Pi}$, the structural consequences for minimizers—such 
as sparsity and tightness of the coupling relaxation—remain unchanged.
\end{remark}

\subsection{What about cyclical monotonicity?}

I now return to the last property left aside: the cyclical monotonicity of optimal transport plans. 
For GW, deriving monotonicity-type results and extending this notion is considerably more challenging. 
Still, I will present an argument that is sometimes used to study optimal GW plans, and discuss its limitations.

The key idea is that a solution of a quadratic program (QP) is also a solution of a suitably associated linear program (LP). 
By analyzing this LP, we can gain insight into the structure of the QP solutions. 
This perspective was used in \citet{vincent2021online} to differentiate the GW distance with respect to the weights $\abf$ and $\bbf$, 
and more recently in \citet{murray_Pickarski} to analyze optimal transport plans in the \emph{semi-relaxed} GW setting, 
in particular to detect when a Monge map exists (i.e., when the coupling relaxation is tight). 

\begin{proposition}
    \label{prop:cyclic_gw}
    Consider a symmetric 4D tensor $\Lbf$. If $\Pbf^\star$ is a solution of \eqref{eq:quad_ot} then it is also a solution of the linear problem
    \begin{equation}
        \min_{\Pbf \in \Pi(\abf, \bbf)} \ \langle \Lbf \otimes \Pbf^{\star}, \Pbf\rangle\,.
    \end{equation}
    In other words, $\Pbf^\star$ solves \eqref{eq:linear_ot} with $\Cbf = \Lbf \otimes \Pbf^{\star}$.
\end{proposition}

\begin{proof}
    A proof can be found in \citet[Theorem 1.12]{murty1988linear} but I write it for completeness. It essentially reduces to proving the first-order optimality conditions in this setting.

We note $f(\Pbf) = \langle \Lbf \otimes \Pbf, \Pbf\rangle$ the GW loss.
Let $\Pbf^{0}$ be a solution of the linear problem \eqref{eq:linear_ot} with $\Cbf = \Lbf \otimes \Pbf^{\star}$. 
We consider for $\lambda \in (0,1)$ the matrix
\begin{equation}
\Pbf^{\lambda} =  \lambda \Pbf^{0} + (1-\lambda) \Pbf^{\star} = \Pbf^{\star}+\lambda(\Pbf^{0}-\Pbf^{\star})\,.
\end{equation}
Then, by convexity of $\Pi(\abf, \bbf)$, we have $\Pbf^{\lambda} \in \Pi(\abf, \bbf)$. Also, since $\Pbf^{\star}$ is optimal, 
\begin{equation}
f(\Pbf^{\star})-f(\Pbf^{\lambda}) \leq 0\,.
\end{equation}
But
\begin{equation}
\begin{split}
f(\Pbf^{\lambda}) &= f(\Pbf^{\star}+\lambda(\Pbf^{0}-\Pbf^{\star})) = \langle \Lbf\otimes \Pbf^{\star}, \Pbf^{\star}\rangle 
+ \langle \Lbf\otimes \Pbf^{\star}, \lambda(\Pbf^{0}-\Pbf^{\star})\rangle\\
&+\langle \Lbf\otimes \lambda(\Pbf^{0}-\Pbf^{\star}), \Pbf^{\star}\rangle 
+\lambda^{2} \langle \Lbf\otimes (\Pbf^{0}-\Pbf^{\star}), (\Pbf^{0}-\Pbf^{\star})\rangle \\
&=f(\Pbf^{\star})+2\lambda \langle \Lbf\otimes \Pbf^{\star}, (\Pbf^{0}-\Pbf^{\star})\rangle +\lambda^{2} \langle \Lbf\otimes (\Pbf^{0}-\Pbf^{\star}), (\Pbf^{0}-\Pbf^{\star})\rangle
\end{split}
\end{equation}
Using $f(\Pbf^{\star})-f(\Pbf^{\lambda}) \leq 0 $ and dividing by $\lambda>0$ implies
\begin{equation}
2\langle \Lbf\otimes \Pbf^{\star}, (\Pbf^{0}-\Pbf^{\star})\rangle 
+\lambda \langle \Lbf\otimes (\Pbf^{0}-\Pbf^{\star}), (\Pbf^{0}-\Pbf^{\star})\rangle \geq 0\,.
\end{equation}
Since this is true for any $\lambda \in (0,1)$, by letting $\lambda \rightarrow 0^{+}$ we obtain
\begin{equation}
2\langle \Lbf\otimes \Pbf^{\star}, (\Pbf^{0}-\Pbf^{\star})\rangle \geq 0 \implies \langle \Lbf\otimes \Pbf^{\star}, \Pbf^{0}\rangle \geq \langle \Lbf\otimes \Pbf^{\star}, \Pbf^{\star}\rangle\,.
\end{equation}
Since $\Pbf^{0}$ is any optimal solution for the linear problem this implies that $\Pbf^{\star}$ is an optimal solution, which concludes the proof.
\end{proof}

The previous result shows that an optimal GW plan is also optimal for a linear OT problem, with the important twist that the cost itself depends on the solution. 
The following small monotonicity-type result for GW optimal plans is a consequence.
\begin{corollary}
    Let $\Pbf^\star$ be optimal for GW and $\Cbf = \Cbf(\Pbf^\star) \triangleq \Lbf \otimes \Pbf^\star$. Then for any  
    for any $N \in \mathbb{N}^{*}, (i_1, j_1),\cdots, (i_N, j_N) \in \supp(\Pbf^\star)^N$ and permutation $\sigma \in \mathfrak{S}_N$,
\begin{equation}
\sum_{k=1}^{N} C_{i_k j_k} \leq \sum_{k=1}^{N} C_{i_k j_{\sigma(k)}}\,. 
\end{equation}
\end{corollary}
This result is mostly a curiosity: in a sense, GW plans exhibit a form of monotonicity, 
but it is not something we can readily exploit. 
Moreover, the converse is false, and since the cost itself depends on the optimal plan, 
it is difficult to derive broad general statements. 
Still, when $\Lbf$ has additional structure, this perspective can be pushed further to obtain meaningful 
information about optimal GW plans \citep{murray_Pickarski}.

\section{Beyond the CND/CPD case: situations where no permutation is optimal}

In this note, I have shown that the CND/CPD property makes it possible to recover GW
counterparts of several classical results of linear OT, in particular the existence of
sparse optimal plans and a ``Monge $=$ Kantorovich'' phenomenon (some optimal plans are
permutations). A natural question is how far one can go beyond the CND/CPD setting.
The examples below show that, in general, the relaxation from permutations to couplings
is not tight. The examples of this section were found with the assistance of a generative AI model
(Claude), and I test them numerically. The code is provided as
supplementary material.
All these examples use $\Lcal(a,b)=(a-b)^2$, so that the GW objective reads
\begin{equation*}
    f(\Pbf) \stackrel{\Delta}{=} \sum_{ijkl}(C_{ik} - \overline{C}_{jl})^2 P_{ij} P_{kl}\,.
\end{equation*}
\begin{example}[A CND/CPD counterexample]
Consider the GW problem with $\Lcal(a,b)=(a-b)^2$ and two discrete measures with
$n=m$ atoms, uniform weights $\abf=\bbf=\frac1n\one_n$, and atoms
$\xbf_1,\dots,\xbf_n\in\R^d$ and $\ybf_1,\dots,\ybf_n\in\R^p$. Assume that both
measures are centered, i.e., $\sum_{i=1}^n \xbf_i=0$ and $\sum_{i=1}^n \ybf_i=0$.
Let $C_{ik}=\|\xbf_i-\xbf_k\|_2^2$, which defines a CND matrix $\Cbf$, and
$\overline C_{jl}=\langle \ybf_j,\ybf_l\rangle$, which defines a positive
semi-definite, hence CPD, matrix $\overline\Cbf$.

In this setting, the GW objective satisfies\footnote{Expanding
$C_{ik}=\|\xbf_i\|^2+\|\xbf_k\|^2-2\langle\xbf_i,\xbf_k\rangle$ and using that
$\Ybf^\top\Pbf^\top\one_n=\Ybf^\top\bbf=0$, all terms involving
$\|\xbf_i\|^2$ vanish, leaving
$f(\Pbf)=\mathrm{cst}+4\|\Xbf^\top\Pbf\Ybf\|_F^2$.}
$f(\Pbf)=\mathrm{cst}+4\|\Xbf^\top\Pbf\Ybf\|_F^2$, where
$\Xbf=(\xbf_1,\dots,\xbf_n)^\top\in\R^{n\times d}$ and
$\Ybf=(\ybf_1,\dots,\ybf_n)^\top\in\R^{n\times p}$, so that optimal couplings are
exactly the minimizers of $\|\Xbf^\top\Pbf\Ybf\|_F^2$. The product coupling
$\Pbf=\abf\bbf^\top=\frac1{n^2}\one_n\one_n^\top$ gives
$\|\Xbf^\top\Pbf\Ybf\|_F^2=\frac1{n^4}\|(\Xbf^\top\one_n)(\Ybf^\top\one_n)^\top\|_F^2=0$
since the measures are centered; it is therefore optimal, and the minimum of $f$
equals $\mathrm{cst}$. Consequently, a permutation $\sigma$ is optimal if and
only if $\sum_{i=1}^n \xbf_i\ybf_{\sigma(i)}^\top=0$, which fails in general.
For instance, take $d=p=1$, $n=2$, $x_1=y_1=-1$ and $x_2=y_2=1$: the identity
permutation gives $\sum_i x_iy_i=2$ and the swap gives $-2$. In this case, \textbf{there is no permutation that leads to an optimal solution}: the unique optimal coupling is the product coupling.
\end{example}

The previous example is special: $\Cbf$ is CND while $\overline{\Cbf}$ is CPD, i.e., the two
cost matrices are of opposite types. This is the regime opposite to the one studied above,
in which the GW objective is convex (rather than concave) on the transport polytope, and
here the product coupling is even optimal. In general, however, cost matrices need not
belong to either regime. I now give an example in the indefinite setting, so that the objective is
neither concave nor convex and none of the previous results applies.
\begin{example}[Two indefinite matrices]
Take $n=m=3$, uniform weights $\abf=\bbf=\frac13\one_3$, $\Lcal(a,b)=(a-b)^2$, and
\begin{equation}
\Cbf=\begin{pmatrix}2&1&1\\1&0&1\\1&1&0\end{pmatrix},
\qquad
\overline{\Cbf}=\begin{pmatrix}2&3&0\\3&2&3\\0&3&1\end{pmatrix}.
\end{equation}
Both matrices are symmetric with nonnegative entries and neither $\Cbf$ nor $\overline{\Cbf}$ is CND or CPD: $\Hbf_3 \Cbf \Hbf_3$ has two nonzero eigenvalues $-1$ and $\frac13$, and 
$\Hbf_3 \overline{\Cbf} \Hbf_3$ has two nonzero eigenvalues $\approx -2.2$ and $\approx 1.5$.
For a permutation $\sigma$, I note $\Pbf_\sigma$ the associated coupling.
A direct computation gives
$\min_\sigma f(\Pbf_\sigma) \approx 2.5556$.
Now, writing $\sigma = (\sigma(1), \sigma(2), \sigma(3))$, consider the coupling
\begin{equation}
\Pbf^\star=\frac16\begin{pmatrix}0&2&0\\1&0&1\\1&0&1\end{pmatrix}
=\tfrac12\,\Pbf_{(2,1,3)}+\tfrac12\,\Pbf_{(2,3,1)}\in\Pi(\abf,\bbf).
\end{equation}
Then
\begin{equation*}
    f(\Pbf^\star) \approx 2.2222 <\min_{\sigma}f(\Pbf_\sigma).
\end{equation*}
Therefore, no permutation coupling is optimal for this GW problem.
\end{example}

Finally, the cost matrices used in practice for GW are typically symmetric, hollow
(zero diagonal) and nonnegative, as for distance matrices. I conclude with an example
of such matrices, which are indefinite and for which no permutation is optimal. Note
that these matrices are not metric (the triangle inequality fails), so whether such an
example exists for genuine distance matrices remains open (to the best of my knowledge).

\begin{example}[Hollow non-negative indefinite matrices]
Let $n=m=4$, $\abf=\bbf=\frac14\one_4$, $\Lcal(a,b)=(a-b)^2$, and
\[
\Cbf=\begin{pmatrix}0&4&8&5\\4&0&1&5\\8&1&0&0\\5&5&0&0\end{pmatrix},\qquad
\overline{\Cbf}=\begin{pmatrix}0&8&6&0\\8&0&4&1\\6&4&0&0\\0&1&0&0\end{pmatrix}.
\]
Both are symmetric, with zero diagonal and nonnegative entries.
$\Hbf_4\Cbf\Hbf_4$ has nonzero eigenvalues $\approx 2.29,-4.78,-9.01$ and
$\Hbf_4\overline{\Cbf}\Hbf_4$ has nonzero eigenvalues $\approx 2.61,-3.69,-8.42$.
Among the $24$ permutation couplings, the smallest GW cost is $7.25 = \frac{29}{4}$, attained for
$\sigma=(1,4,2,3)$. Let $\sigma_1=(1,3,2,4)$ and $\sigma_2=(1,4,3,2)$, and
$\Pbf(t)=(1-t)\Pbf_{\sigma_1}+t\Pbf_{\sigma_2} \in \Pi(\abf, \bbf)$. Then
$f(\Pbf(t))$ is a degree two polynomial in $t$, whose minimum in $[0,1]$ is attained at $t \approx 0.5417$, corresponding to the value $\approx 7.1198 <\min_\sigma f(\Pbf_\sigma)$.
Hence, no permutation is optimal.
\end{example}

Although such counterexamples do exist, the main point of \citet[Section 3]{maron2018probably} is precisely to argue that such situations appear to be uncommon in practice, as CND-type energies arise quite frequently.

\bibliographystyle{unsrtnat}
\bibliography{references}

\end{document}